\documentclass{article}

\usepackage[english]{babel}
\usepackage{natbib}

\usepackage{amsmath}
\usepackage{mathtools}
\usepackage{graphicx}
\usepackage[colorlinks=true, allcolors=blue]{hyperref}
\usepackage[utf8]{inputenc}
\usepackage{amssymb}
\usepackage{amsthm}
\usepackage{hyperref}
\usepackage{cleveref}
\usepackage{bbm}
\usepackage{fullpage}

\newtheorem{theorem}{Theorem}

\newtheorem{lemma}{Lemma}[theorem]

\title{A Note on the Chernoff Bound for Random Variables in the\\ Unit Interval}
\author{Andrew Y.~K.~Foong\textsuperscript{1}, Wessel P.~Bruinsma\textsuperscript{1, 2}, David R.~Burt\textsuperscript{1}\\
\textsuperscript{1}University of Cambridge, \textsuperscript{2}Invenia Labs\\
\texttt{\{ykf21, wpb23, drb62\}@cam.ac.uk}}
\date{May 2022}

\begin{document}

\maketitle

\begin{abstract}
    The Chernoff bound is a well-known tool for obtaining a high probability bound on the expectation of a Bernoulli random variable in terms of its sample average. This bound is commonly used in statistical learning theory to upper bound the generalisation risk of a hypothesis in terms of its empirical risk on held-out data, for the case of a binary-valued loss function. However, the extension of this bound to the case of random variables taking values in the unit interval is less well known in the community. In this note we provide a proof of this extension for convenience and future reference.
\end{abstract}

\section{Introduction}

In statistical learning theory, 
one commonly considers a hypothesis space $\mathcal{H}$ and a probability measure $D$ over a space of datapoints $\mathcal{Z}$.
Let $\ell\colon \mathcal{H} \times \mathcal{Z} \to \{0, 1\}$ be a loss function taking values in $\{0, 1\}$;
\textit{i.e.}, $\ell$ is a binary loss.
It is often of interest to bound the generalisation risk $R(h) \coloneqq \mathbbm{E}_{Z \sim D}[\ell(h, Z)]$ of a hypothesis $h \in \mathcal{H}$.
Such a bound can be established by computing the empirical mean of the loss on a dataset $S \sim D^N$ and using the following well-known theorem:

\begin{theorem}[Chernoff bound for binary random variables; \citealp{langford2005tutorial}, Corollary 3.7]
\label{thm:chernoff-bound-binary}
Let $X_1, \hdots, X_n$ be i.i.d.~random variables with $X_i \in \{0, 1\}$ and $\mathbbm{E}[X_i] = p$. Let $\overline{X} := \frac{1}{n} \sum_{i=1}^n X_i$.
Then, with probability at least $1-\delta$,
\begin{equation}
    p \leq \mathrm{kl}^{-1}\left(\overline{X} \, \middle| \, \frac{1}{n} \log \frac{1}{\delta} \right),
\end{equation}
where $\mathrm{kl}(q,p) \coloneqq q \log \frac{q}{p} + (1-q) \log \frac{1-q}{1-p}$ and $\mathrm{kl}^{-1}(q \,|\, c) \coloneqq \sup \, \{ p \in [0,1] : \mathrm{kl}(q, p) \leq c \}$.
\end{theorem}

By choosing $X_i = \ell(h, Z_i)$ to be the loss of hypothesis $h$ on the $i$\textsuperscript{th} datapoint $Z_i$ in $S$, we immediately obtain a high-probability upper bound on the generalisation risk, assuming that $h$ does not depend on $S$. (\textit{I.e.}, if $h$ is the result of training algorithm which uses training data, then $S$ must be held-out data.) 

\Cref{thm:chernoff-bound-binary} is well-known and, as stated, requires $X_i$ to be Bernoulli random variables.
However, if the loss $\ell$ takes values in the unit interval $[0, 1]$ rather than in $\{0, 1\}$, \textit{i.e.} $\ell: \mathcal{H} \times \mathcal{Z} \to [0, 1]$, then we require an extension of \cref{thm:chernoff-bound-binary} to the case of $X_i$ taking values in $[0, 1]$:


\begin{theorem}[Chernoff bound for random variables in the unit interval]
\label{thm:chernoff-bound-interval}
Let $X_1, \hdots, X_n$ be i.i.d.~random variables with $X_i \in [0,1]$ and $\mathbbm{E}[X_i] = p$.
Then, using the same notation as in \cref{thm:chernoff-bound-binary}, with probability at least $1-\delta$, 
\begin{align}
    p \leq \mathrm{kl}^{-1}\left(\overline{X} \, \middle| \, \frac{1}{n} \log \frac{1}{\delta} \right).
\end{align}
\end{theorem}

Although losses $\ell$ taking values in $[0, 1]$ are commonly encountered, \cref{thm:chernoff-bound-interval} is somewhat less well known than \cref{thm:chernoff-bound-binary}.
In this note, we provide a proof of \cref{thm:chernoff-bound-interval} for convenience and future reference.


\section{Two-Sided Chernoff Bound}

We first recapitulate some well-known bounds based on Hoeffding's extension of the Chernoff bound:

\begin{lemma}[Hoeffding's extension]\label{lem:hoeffding}
Let $X_1, \hdots, X_n$ be i.i.d.~random variables with $X_i \in [0,1]$ and $\mathbbm{E}[X_i] = p$. Then for any $t \in [0, p]$, 
\begin{align} \label{eqn:chernoff-up}
    \mathrm{Pr}(\overline{X} \leq p - t) &\leq \exp(-n\mathrm{kl}(p-t, p)), 
\end{align}
and for any $t \in [0, 1-p]$,
\begin{align} \label{eqn:chernoff-down}
    \mathrm{Pr}(\overline{X} \geq p + t) &\leq \exp(-n\mathrm{kl}(p+t, p)).
\end{align}
\end{lemma}
\begin{proof}
The proof is an 
application of the Chernoff method along with the observation that $z \mapsto e^{\lambda z}$ is convex, which allows us to control the moment-generating function in terms of the moment-generating function of a Bernoulli random variable. A detailed proof of the upper bound in \cref{eqn:chernoff-up} is given in Theorem 5.1 of \cite{mulzer2018five}. \Cref{eqn:chernoff-down} follows by applying the change of variables $X_i \to 1-X_i$ to \cref{eqn:chernoff-up} (see Corollary 4.1 in \cite{mulzer2018five} for an identical change of variables argument in the binary case).
%
\end{proof}

We can use \cref{lem:hoeffding} to obtain a \emph{two-sided} bound on the mean $p$.

\begin{theorem}[Two-sided Chernoff bound for random variables in the unit interval]\label{thm:chernoff-interval-two-sided}
Let $X_1, \hdots, X_n$ be i.i.d.~random variables with $X_i \in [0,1]$ and $\mathbbm{E}[X_i] = p$.
Then, with probability at least $1-\delta$,
\begin{align}
    \mathrm{kl}\left(\overline{X},p \right) \le \frac{1}{n} \log \frac{2}{\delta}.
\end{align}
\end{theorem}
\begin{proof}
For $c > 0$, let $\overline{\mathrm{kl}} (p, c)$ be the unique real number in $(p, 1]$ such that $\mathrm{kl} (\overline{\mathrm{kl}} (p, c),  p) = c$. Similarly, let $\underline{\mathrm{kl}} (p, c)$ be the unique real number in $[0, p)$ such that  $\mathrm{kl} (\underline{\mathrm{kl}} (p, c), p) = c$. 
Then, $\mathrm{kl}(\overline{X} , p) \leq c$ if and only if $\underline{\mathrm{kl}} (p, c) \leq \overline{X} \leq \overline{\mathrm{kl}} (p, c)$. Hence,
\begin{equation}
    \mathrm{Pr}\left(\mathrm{kl}(\overline{X} , p) \leq \frac{1}{n} \log \frac{2}{\delta} \right) = \mathrm{Pr}\left( \underline{\mathrm{kl}} \left(p, \frac{1}{n} \log \frac{2}{\delta} \right) \leq \overline{X} \leq \overline{\mathrm{kl}} \left(p, \frac{1}{n} \log \frac{2}{\delta} \right) \right). 
\end{equation}
But, from \cref{eqn:chernoff-up}, we have
\begin{equation}
    \mathrm{Pr} \left(\underline{\mathrm{kl}} \left(p, \frac{1}{n} \log \frac{2}{\delta} \right) \leq \overline{X} \right)
    \geq
    1 - \exp \left(-n\mathrm{kl}\left( \underline{\mathrm{kl}} \left(p, \frac{1}{n} \log \frac{2}{\delta} \right) ,  p\right) \right)
    = 1 - \frac{\delta}{2}.
\end{equation}
A symmetric argument shows that
\begin{equation}
    \mathrm{Pr}\left(\overline{X} \leq \overline{\mathrm{kl}} \left(p, \frac{1}{n} \log \frac{2}{\delta} \right) \right) \geq 1 - \frac{\delta}{2}.
\end{equation}
Using a union bound therefore implies that
\begin{equation}
    \mathrm{Pr}\left(\mathrm{kl}(\overline{X} , p) \leq \frac{1}{n} \log \frac{2}{\delta} \right) \geq 1 - \delta,
\end{equation}
which proves the theorem.
\end{proof}

This \emph{two-sided} bound on the mean $p$ nearly implies the desired \emph{one-sided} bound from \cref{thm:chernoff-bound-interval}:
\begin{equation} \label{eq:implication}
    \mathrm{kl}\left(\overline{X} \, \middle| \,p \right) \le \frac{1}{n} \log \frac{2}{\delta}
    \implies 
    p \leq \mathrm{kl}^{-1}\left(\overline{X} \, \middle| \, \frac{1}{n} \log \frac{2}{\delta} \right).
\end{equation}
Unfortunately, on the right-hand side, \cref{eq:implication} states $2/\delta$ instead of $1/\delta$.

\section{One-Sided Chernoff Bound}

We can tighten the $2/\delta$ in \cref{eq:implication} to obtain the $1/\delta$ from \cref{thm:chernoff-bound-interval} by directly proving a one-sided bound on the mean $p$.

\begin{proof}
The main ingredient of the proof of \cref{thm:chernoff-interval-two-sided} is the equivalence $\mathrm{kl}(\overline{X}, p) \leq c$ if and only if $\underline{\mathrm{kl}} (p, c) \leq \overline{X} \leq \overline{\mathrm{kl}} (p, c)$.
Since we now desire only an upper bound on $p$, $\mathrm{kl}(\overline{X}, p) \leq c$ is stronger than we need.
The key insight is to define a \emph{one-sided} version of $\mathrm{kl}$:
\begin{equation}
        \mathrm{kl}_m(q, p)= 
    \begin{cases}
        \mathrm{kl}(q, p ),& q\leq p.\\
        0,              & q > p.
    \end{cases}
\end{equation}
An upper bound on $p$ is then equivalent to an upper bound on $\mathrm{kl}_m(\overline{X}, p)$:
\begin{equation}
    p \le \mathrm{kl}^{-1}\left(\overline{X}\,\middle|\, \frac{1}{n} \log \frac{1}{\delta}\right)
    \iff \mathrm{kl}_m(\overline{X}, p) \le \frac{1}{n} \log \frac{1}{\delta}.
\end{equation}
To see this, note that
\begin{equation}
    \sup \left\{  p \in [0,1] : \mathrm{kl}(\overline{X}, p ) \leq \frac{1}{n} \log \frac{1}{\delta} \right\}
    = \sup \left\{  p \in [0,1] : \mathrm{kl}_m(\overline{X}, p ) \leq \frac{1}{n} \log \frac{1}{\delta} \right\}
\end{equation} 
because, in both suprema, $\overline{X} \le p$ always. The same approach is taken in the proof of the one-sided bound for Bernoulli random variables in \cref{thm:chernoff-bound-binary}, see the proof of Lemma 3.6 in \cite{langford2005tutorial}.

Analogously to the proof of \cref{thm:chernoff-interval-two-sided}, the main ingredient of this proof is the equivalence $\mathrm{kl}_m(\overline{X}, p ) \leq c$ if and only if $\underline{\mathrm{kl}}(p,c) \leq \overline{X}$, so
we conclude by the fact that the latter holds with probability at least $1 - \delta$:
\begin{equation}
    \mathrm{Pr} \left( \mathrm{kl}_m(\overline{X} , p ) \leq \frac{1}{n} \log \frac{1}{\delta}  \right) = \mathrm{Pr} \left( \underline{\mathrm{kl}}\left(p, \frac{1}{n} \log \frac{1}{\delta} \right) \leq \overline{X} \right)
    \geq 1 - \delta.
\end{equation}
%
\end{proof}

\bibliographystyle{apalike}
\bibliography{sample}

\end{document}